\documentclass[10pt,twocolumn,letterpaper]{article}

\usepackage{wacv}
\usepackage{times}
\usepackage{epsfig}
\usepackage{graphicx}
\usepackage{amsmath}
\usepackage{amssymb}
\usepackage{booktabs}

\usepackage{url}

\usepackage[utf8]{inputenc} 
\usepackage[T1]{fontenc}    
\usepackage{url}            
\usepackage{booktabs}       
\usepackage{amsfonts}       
\usepackage{nicefrac}       
\usepackage{microtype}      
\usepackage{xcolor}         
\usepackage{graphicx}
\usepackage{xcolor,colortbl}
\usepackage{multirow}
\usepackage{caption}
\usepackage{wrapfig}
\usepackage{subcaption}
\usepackage[inline]{enumitem}
\usepackage{amsthm}

\usepackage{mathtools}

\definecolor{Gray}{gray}{0.85}
\definecolor{LightCyan2}{rgb}{0.94,0.81,0.81}

\definecolor{LightCyan}{rgb}{0.84,0.84,0.84}

\newtheorem{prop}{Proposition}

\def\E{{\rm E}\,}

\def\wacvPaperID{503} 

\wacvalgorithmstrack   

\wacvfinalcopy 

\ifwacvfinal
\usepackage[breaklinks=true,bookmarks=false]{hyperref}
\else
\usepackage[pagebackref=true,breaklinks=true,colorlinks,bookmarks=false]{hyperref}
\fi

\begin{document}

\title{
Compact and Optimal Deep Learning with Recurrent Parameter Generators
}

\author{%
\setlength{\tabcolsep}{10pt}
\begin{tabular}{@{}ccccc@{}}
Jiayun Wang\thanks{indicates equal contribution.
}$^{\ \ 1}$ &  
Yubei Chen$^{*\ 3,4}$ &
Stella X. Yu$^{1,2}$&
Brian Cheung$^{5}$ &
Yann LeCun$^{3,4}$ \\
\end{tabular}\\[5pt]
\setlength{\tabcolsep}{3pt}
\begin{tabular}{@{}ccccc@{}}
$^1$UC Berkeley / ICSI &  
$^2$University of Michigan &
$^3$Meta AI&
$^4$New York University &  
$^5$MIT CSAIL \& BCS \\
\end{tabular}\\
\tt\small\{peterwg,stellayu\}@berkeley.edu \hfill \{yubeic,yann\}@fb.com \hfill cheungb@mit.edu\\
}
\maketitle
\thispagestyle{empty}

\begin{abstract}
Deep learning has achieved tremendous success 
by training increasingly large models, which are then compressed for practical deployment.  We propose a drastically different approach to compact and optimal deep learning:   
 We decouple the Degrees of freedom (DoF) and the actual number of parameters of a model, optimize a small DoF with predefined random linear constraints for a large model of an arbitrary architecture, in one-stage end-to-end learning.
 
Specifically, we create a recurrent parameter generator (RPG), which repeatedly fetches parameters from a ring  and unpacks them onto a large model with random permutation and sign flipping to promote parameter decorrelation.  We show that gradient descent can automatically find the best model under constraints with in fact faster convergence.

Our extensive experimentation reveals a log-linear relationship between model DoF and accuracy.  Our
 RPG demonstrates remarkable DoF reduction, 
and can be further pruned and quantized for additional  run-time performance gain.
For example, in terms of top-1 accuracy on ImageNet,
RPG achieves 96\% of ResNet18's performance with only 18\% DoF (the equivalent of one convolutional layer) and 52\% of ResNet34's performance with only 0.25\% DoF!
Our work shows significant potential of constrained neural optimization in compact and optimal deep learning.
\end{abstract}

\def\figModelRing#1{
\begin{figure}[#1]\centering
\vspace{-1.5em}
\includegraphics[width=0.92\columnwidth]{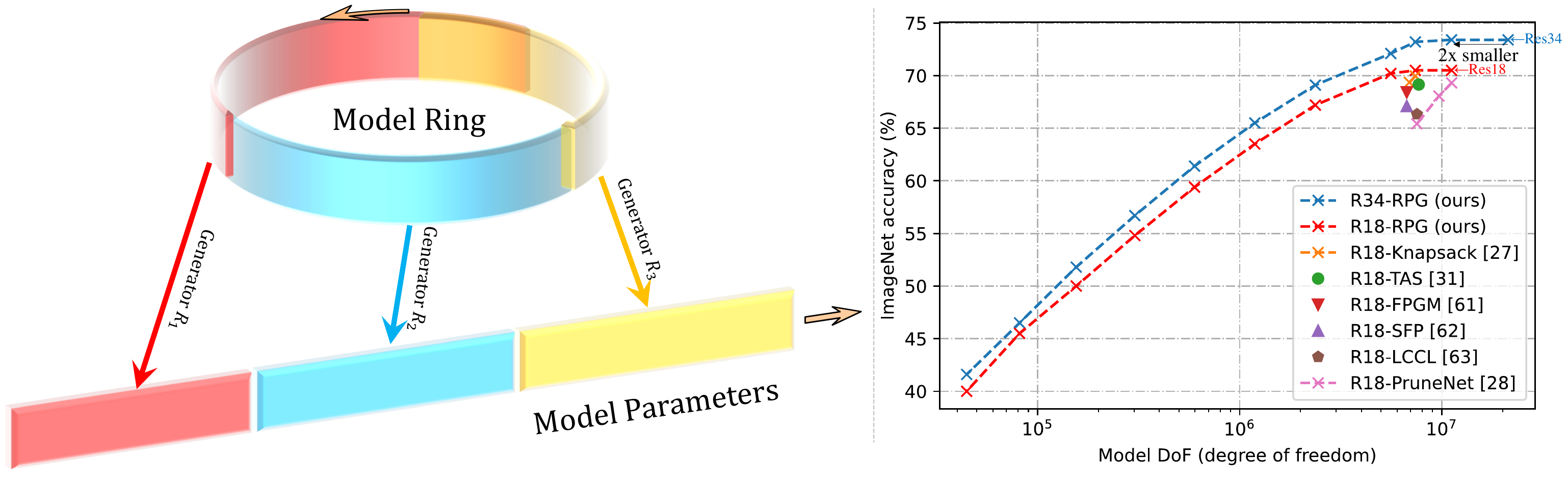}
\caption{ \small We propose a recurrent parameter generator (RPG) that shares a fixed set of parameters in a ring and use them to generate parameters of different parts of a neural network, 
whereas in the standard neural network, all the parameters are independent of each other, so the model gets bigger as it gets deeper. \textbf{Left}: The \textcolor{orange}{third} section of the model starts to overlap with the \textcolor{red}{first} section in the model ring, and all later layers 
 share generating parameters for possibly multiple times.
 \textbf{Right:} Employing the Recurrent Parameter Generator (RPG) for ResNet could reduce the model parameters to any size. Specifically, with only half ResNet34 backbone parameters, we achieve the same ImageNet top-1 accuracy. 
 We also outperform model compression methods such as Knapsack \cite{aflalo2020knapsack}.
}
\vspace{-1.5em}
\label{fig:modelRing}
\end{figure}
}

\def\figtasks#1{
\vspace{-1.5em}
\begin{figure*}[#1]\centering
\includegraphics[width=0.9\textwidth]{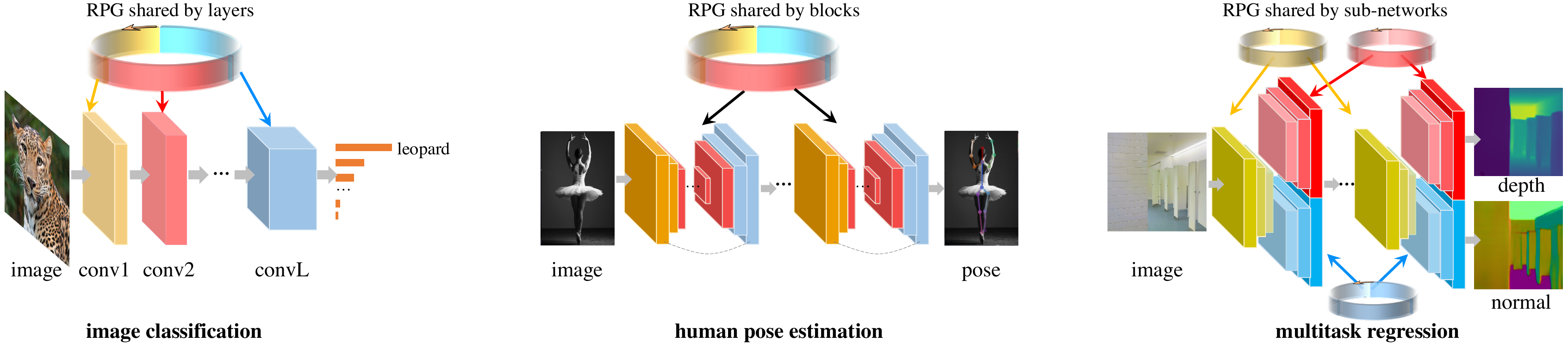}
\vspace{-1em}
\caption{ \small  We demonstrate the effectiveness of RPG on various applications including image classification (\textbf{Left}), human pose estimation (\textbf{Middle}), and multitask regression (\textbf{Right}). RPGs are shared at multiple scales: a network can either have a global RPG or multiple local RPGs that are shared within blocks or sub-networks.
}
\vspace{-1.5em}
\label{fig:tasks}
\end{figure*}
}

\section{Introduction}

\begin{figure}[t]
    \centering
 \includegraphics[width=\columnwidth]{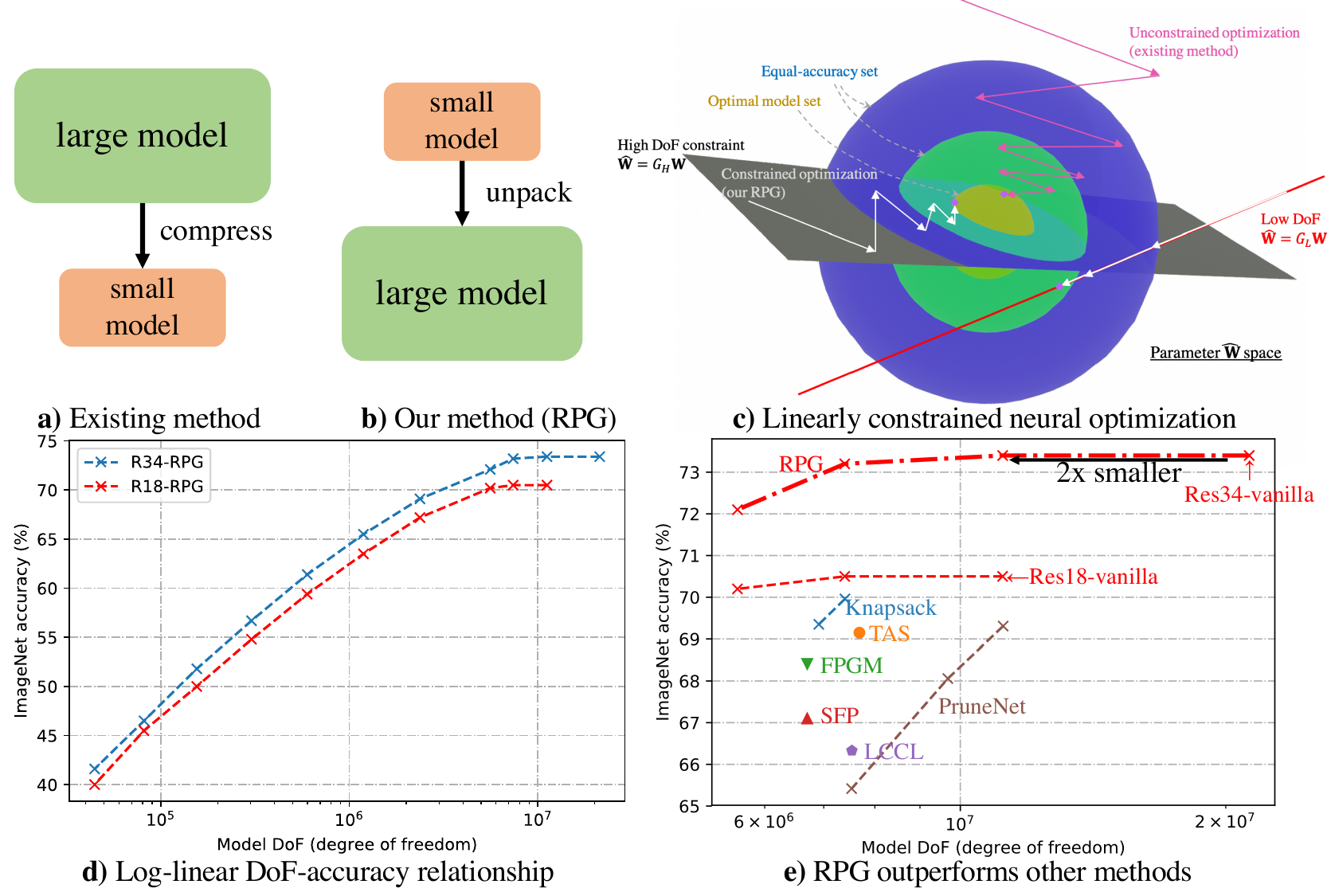}
   \caption{ \small 
   We propose a novel approach to compact and optimal deep learning by decoupling model DoF and model parameters.
   {\bf a}) Existing methods first finds the optimal in a large model space and then compress it for practical deployment.  
   {\bf b})  We propose to start with a small (DoF) model of free parameters, use 
   recurrent parameter generator (RPG) to unpack them onto a large model with predefined random linear projections. 
   {\bf c}) Gradient descent finds the optimal model of a small DoF under these linear constraints with faster converge than training the large unpacked model itself (Fig.\ref{fig:combo2}{\bf b}). 
   If the DoF is too small, the optimal large model may fall out of the constrained subpsace.  However, at a sufficiently large DoF,  RPG gets rid of redundancy and often finds a model with little loss in accuracy. 
    {\bf d}) RPG reveals a log-linear relationship between model DoF and accuracy.
   {\bf e}) RPG achieves the same ImageNet accuracy with half of the ResNet-vanilla DoF. RPG also outperforms other state-of-the-art compression approaches.
   }  
  \label{fig:teaser}
  \vspace{-1.5em}
\end{figure}

Deep neural networks as general optimization tools have achieved great success with increasingly more training data, deeper and larger neural networks:  A recently developed NLP model, GPT-3 \cite{brown2020language}, has astonishing 175 billion parameters!  
While the model performance generally scales with the number of parameters \cite{henighan2020scaling}, with parameters outnumbering training data, the model is significantly over-parameterized. 

Many approaches have been proposed to remove redundancy in trained large models: neural network pruning \cite{lecun1990optimal, han2015deep, liu2018rethinking}, efficient network design spaces \cite{howard2017mobilenets, iandola2016squeezenet, sandler2018mobilenetv2}, parameter regularization \cite{wan2013regularization, wang2020orthogonal, srivastava2014dropout, NowlanH92simplifying}, model quantization \cite{hubara2017quantized, rastegari2016xnor, louizos2018relaxed}, neural architecture search \cite{zoph2016neural, cai2018proxylessnas, wan2020fbnetv2}, recurrent models \cite{bai2019deep, bai2020multiscale, wei2016convolutional}, multi-task feature encoding \cite{ramamonjisoa2019sharpnet, hao2020multi}, etc. 
Pruning-based model compression dates back to the late 80s \cite{mozer1989using, lecun1990optimal} and has enjoyed recent resurgence \cite{han2015deep, blalock2020state}. They remove unimportant parameters from a pre-trained model and can achieve significant model compression.

 \begin{figure}[t]
    \centering
 \includegraphics[width=\columnwidth]{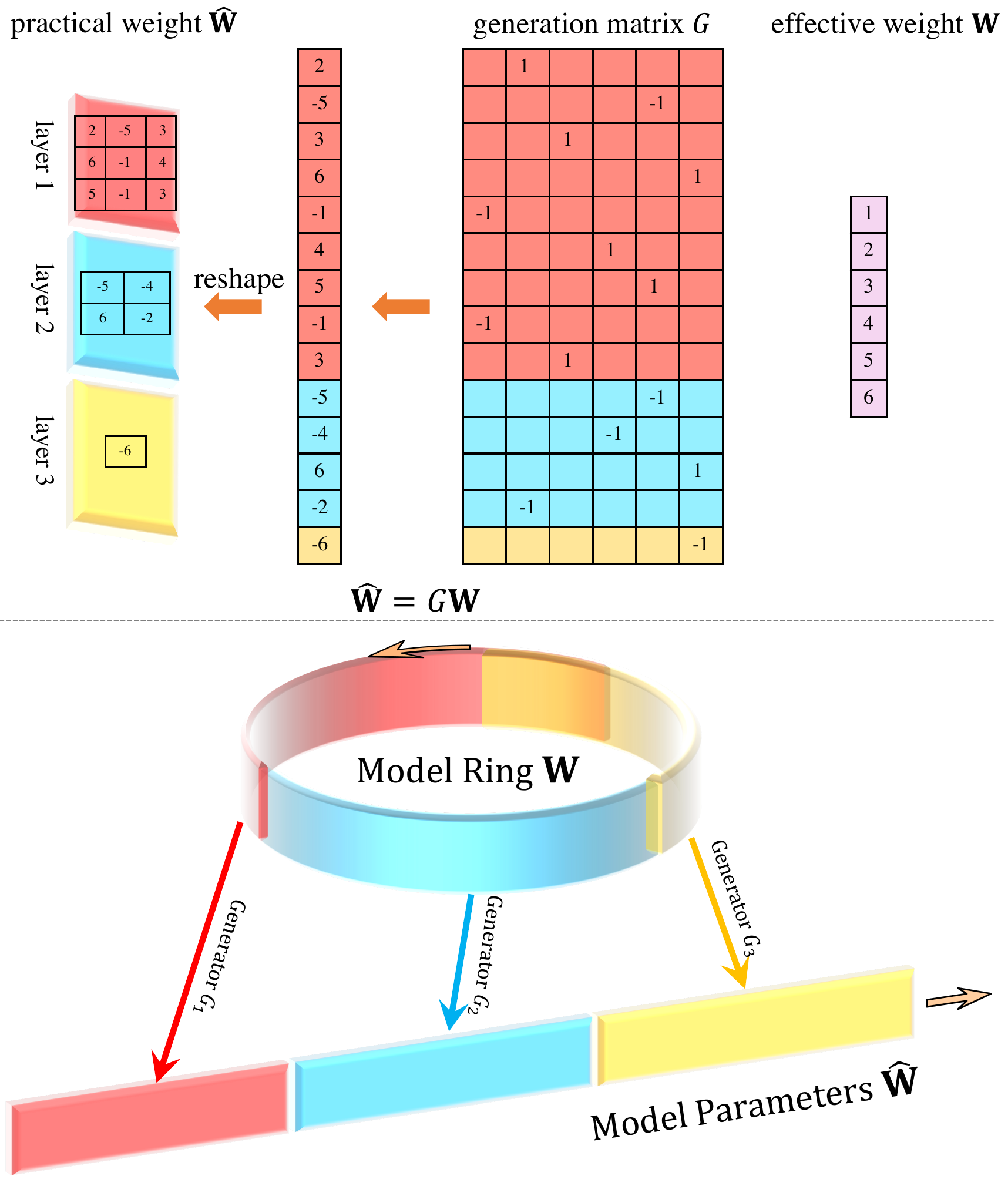}
   \vspace{-1em}
   \caption{ \small
   {\bf Upper:} Networks are optimized with a linear constraint $\hat{\mathbf{W}} = \boldsymbol{G} \mathbf{W}$, where the constrained parameter $\hat{\mathbf{W}}$ of each network layer was generated by the generating matrix $\boldsymbol{G}$ from the free parameter $\mathbf{W}$, which is directly optimized. $\hat{\mathbf{W}}$ is unpacked large model parameter while the size of $\mathbf{W}$ is the model DoF.
   {\bf Lower:} This paper discusses a specific format of parameter generation, 
  recurrent parameter generator (RPG).
   RPG shares a fixed set of parameters in a ring and uses them to generate parameters of different parts of a neural network,
whereas in the standard neural network, all the parameters are independent of each other, so the model gets bigger as it gets deeper. The \textcolor{orange}{third} section of the model starts to overlap with the \textcolor{red}{first} section in the model ring, and all later layers 
 share generating parameters for possibly multiple times.
   }  
  \label{fig:modelring_}
  \vspace{-1.5em}
\end{figure}

Our work is a departure from mainstream approaches towards model optimization and parameter reduction: rather than compressing a large model, we directly optimize a lean model with a small set of free parameters (number of free parameters equal to degree of freedom of the model, or DoF), which can be linearly unpacked to a large model. Training the large model can be viewed as solving a neural optimization with a set of predefined linear constraints. One benefit of constrained neural optimization we observe is that it leads to a faster convergence rate (Section \ref{sec:general}).
Specifically, we define different layers in a neural network based on a fixed amount of DoF, which we call \textit{recurrent parameter generator} (RPG).
That is, we differentiate the number of model parameters and DoF.
Traditionally, model parameters are treated independently of each other; the total number of parameters equals DoF. However, by tapping into how a core set of free parameters can be assigned to the neural network model, we can develop a large model of many parameters, which are linearly constrained by the small set of free parameters.

There is excess capacity in neural networks independent of how and where the parameters are used in the network, even at the level of individual scalar values.
Surprisingly, backpropagation training of a deep network is able to cope with that the same parameter can be assigned to multiple random locations in the network without significantly impacting model performance. Our extensive experiments show that a large neural network does not need to be over-parameterized to achieve competitive performance. Particularly, a ResNet18 can be implemented with DoF equivalent to one convolution layer in a ResNet18-vanilla ($4.72\times$ DoF reduction) and still achieves $67.2\%$ ImageNet top-1 accuracy. The proposed method is also extremely flexible in reducing model DoF. In some sense, the proposed RPG method can be viewed as an automatic model DoF reduction technique, which explores the optimal accuracy-parameter trade-off. When we reduce the model DoF, RPG demonstrates graceful performance degradation, and its compression results are frequently on par with the SOTA pruning methods besides the flexibility. Even if we reduce the Res18 backbone DoF to $36$K, which is about $300\times$ reduction, ResNet18 can still achieve $40\%$ ImageNet top-1 accuracy. Further, we show RPG can be quantized and pruned to improve FLOPs and runtime with relatively
mild accuracy drops. 

To summarize, we make three contributions:
\begin{enumerate*}
    \item We provide a new perspective towards automatic model size reduction: we define a neural network with certain DoF with random linear constraints. We discover that gradient descent can automatically solve constrained optimization for the best model with a faster convergence rate. This constrained neural optimization perspective is likely to benefit many other applications. 
    
    \item We propose the recurrent parameter generator (RPG), which decouples the network architecture and the network DoF. We can flexibly choose any desired DoF to construct the network given a specific neural network architecture. 

    \item By separating network architectures from parameters, RPG becomes a tool to understand the relationship between the model DoF and the network performance.  We observe an empirical log-linear DoF-Accuracy relationship.
\end{enumerate*}
\figtasks{t}

\vspace{-0.5em}
\section{Related Work}
\vspace{-0.5em}
Many works study model DoF reduction or compression.
We discuss each one and its relationship to our work.

\noindent \textbf{Model Pruning, Neural Architecture Search, and Quantization.}
Model pruning seeks to remove unimportant parameters in a trained model. Recently, it's proposed to use neural architecture search as coarse-grained model pruning \cite{yu2018slimmable, dong2019network}. Another related effort is  network quantization \cite{hubara2017quantized, rastegari2016xnor, louizos2018relaxed}, which seeks to reduce the bits used for each parameter and can frequently reduce the model size by $4\times$ with minimal accuracy drop. More recently, \cite{dollar2021fast} presents a framework for analyzing model scaling strategies that consider network properties such as FLOPs and activations.

\noindent \textbf{Parameter Regularization and Priors}. 
Regularization has been widely used to reduce model redundancy \cite{krogh1992simple, NowlanH92simplifying}, alleviate overfitting \cite{srivastava2014dropout, wan2013regularization}, and ensure desired mathematical regularity \cite{wang2020orthogonal}. RPG can be viewed as a parameter regularization in the sense that weight sharing poses many equality constraints to  weights and regularizes  weights to a low-dimensional space. HyperNeat \cite{stanley2009hypercube} and CPPNs \cite{stanley2007compositional} use networks to determine the weight between two neurons as a function of their positions. \cite{karaletsos2018probabilistic, KaraletsosB20} introduced a similar idea by providing a hierarchical prior for network parameters.

\noindent \textbf{Recurrent Networks and Deep Equilibrium Models.} Recurrence and feedback have been shown in psychology and neuroscience to act as modulators or competitive inhibitors to aid feature grouping \cite{gilbert2007brain}, figure-ground segregation \cite{hupe1998cortical} and object recognition \cite{wyatte2012limits}. Recurrence-inspired mechanisms also achieve success in feed-forward models. There are two main types of employing recurrence based on if weights are shared across recurrent modules. ResNet \cite{he2016deep}, a representative of reusing similar structures without weight sharing, introduces parallel residual connections and achieves better performance by going deeper in networks. Similarly, some works \cite{szegedy2015going,srivastava2015highway} also suggest iteratively injecting thus-far representations to the feed-forward network useful. Stacked inference methods \cite{ramakrishna2014pose,wolpert1992stacked,weiss2010structured} are also related while they consider each output in isolation. Some find sharing weights across recurrent modules valuable. They demonstrate applications in temporal modelling \cite{weiss2010structured,shi2015convolutional,karpathy2015deep}, spatial attention \cite{mnih2014recurrent,butko2009optimal}, pose estimation \cite{wei2016convolutional,carreira2016human}, and so on \cite{li2016iterative,zamir2017feedback}. Such methods usually shine in modeling long-term dependencies. In this work, we recurrently share weights across different layers of a feedback network to reduce network redundancy. 

Given stacking weight-shared modules improve the performance, researchers consider running even infinite depth of such modules by making the sequential modules converge to a fixed point \cite{lecun1988theoretical,bai2019deep}. Employing such \textit{equilibrium} models to existing networks, they show improved performance in many natural language processing \cite{bai2019deep} and computer vision tasks \cite{bai2020multiscale,wang2020implicit}. One issue with deep equilibrium models is that the forward and backward propagation usually takes much more iterations than explicit feed-forward networks. Some work \cite{fung2021fixed} improves the efficiency by making the backward propagation Jacobian free. Another issue is that \textit{infinite} depth and fixed point may not be necessary or even too strict for some tasks. 
Instead of achieving infinite depth, our model shares parameters to a certain level. We empirically compare with equilibrium models in Section \ref{sec:exp}.

\noindent \textbf{Efficient Network Space and Matrix Factorization.}
Convolution is an efficient and structured matrix-vector multiplication. Arguably, the most fundamental idea in building efficient linear systems is matrix factorization. Given the redundancy in deep convolutional neural network parameters, one can leverage the matrix factorization concept, e.g., factorized convolutions, and design more efficient network classes \cite{howard2017mobilenets, iandola2016squeezenet, tan2019efficientnet, sandler2018mobilenetv2}.

\section{Recurrent Parameter Generator}

\label{sec:method:ls}

\noindent {\bf Linearly Constrained Neural Optimization.} 
Consider optimizing a network with input data $\mathbf{X}$, parameters $\hat{\mathbf{W}}$ and loss function $L$. The optimization can be written as: 
\begin{align}
    \min L(\mathbf{X}; \hat{\mathbf{W}} )\ \ \text{s.t.}\  \mathbf{\hat{\mathbf{W}}} = \boldsymbol{G} \mathbf{W} (\text{or equally} \ \boldsymbol{R}\mathbf{\hat{\mathbf{W}}} = 0)
\end{align}
where $\mathbf{\hat{\mathbf{W}}} = \boldsymbol{G} \mathbf{W}$ refers to a set of linear constraints, where $\boldsymbol{G}\in \Re^{N\times M}$ is a full-rank tall matrix (i.e. $N\ge M$). Here we refer to $\hat{\mathbf{W}}$ as the constrained parameters and $\mathbf{W}$ as the free parameters. This constraint is a change of variable, i.e., the constrained parameter $\hat{\mathbf{W}}$ is linearly generated from the free parameter $\mathbf{W}$ by generating matrix $\boldsymbol{G}$. We can consider $\mathbf{W}$ as a compressed model, which is unpacked into $\hat{\mathbf{W}}$ to construct the large neural network. $\mathbf{W}$ is directly optimized via gradient descent and free to update. In this linearly constrained neural optimization, the model DoF is equivalent to $M$, which is the dimension of $\mathbf{W}$. An equivalent form of the constraint $\mathbf{\hat{\mathbf{W}}} = \boldsymbol{G} \mathbf{W}$ is $\boldsymbol{R}\mathbf{\hat{\mathbf{W}}} = 0$, where $\boldsymbol{R}\in\Re^{(N-M)\times N}$ can be derived from SVD of $\boldsymbol{G}$.

\noindent {\bf Recurrent Parameter Generator.} Let's assume that we construct a deep convolutional neural network containing $L$ different convolution layers. Let $\mathbf{K}_1,\mathbf{K}_2,\dots, \mathbf{K}_L$ be the corresponding $L$ convolutional kernels\footnote{A kernel contains all the filters of one layer. In this paper, we treat each convolutional kernel as a vector. When the kernel is used to do the convolution, it will be reshaped into the corresponding shape.}. 
Rather than using separate sets of parameters for different convolution layers, we create a single set of parameters $\mathbf{W}\in\Re^M$ and use it to generate the corresponding parameters $\hat{\mathbf{W}} = \left[\mathbf{K}_1^T,\mathbf{K}_2^T,\dots, \mathbf{K}_L^T\right]^T \in \Re^N$ for each convolution layer:
\begin{equation}
\mathbf{K}_i = \boldsymbol{G}_i\cdot \mathbf{W}, i \in \{1,\dots,L\}
\label{eq:k}
\end{equation}
where $\boldsymbol{G}_i$ is a fixed predefined generating matrix, which is used to generate $\mathbf{K}_i$ from $\mathbf{W}$.
We call $\boldsymbol{G} = \left[\boldsymbol{G}_1^T,\dots,\boldsymbol{G}_L^T\right]^T$ and $\mathbf{W}$ the \textit{recurrent parameter generator} (RPG). In this work, we always assume that the size of $\mathbf{W}$ is not larger than the total parameters of the model, i.e., $|\mathbf{W}|\leq\sum_{i}{|\mathbf{K}_i|}$. This means an element of $\mathbf{W}$ will generally be used in more than one layer of a neural network. Additionally, the gradient of $\mathbf{W}$ is a linear superposition of the gradients from each convolution layer.
During the neural network training, let's assume convolution kernel $\mathbf{K}_i$ receives gradient $\frac{\partial \ell}{\partial \mathbf{K}_i}$, where $\ell$ is the loss function. Based on the chain rule, it is clear that the gradient of $\mathbf{W}$ is:
\begin{equation}
\frac{\partial \ell}{\partial \mathbf{W}} = \sum_{i=1}^{L}{\boldsymbol{G}^T_i\cdot \frac{\partial \ell}{\partial \mathbf{K}_i}}
\label{eq:grad_w}
\end{equation}
\noindent {\bf Generating Matrices and Destructive Weight Sharing.} There are various ways to create the generating matrices $\{\boldsymbol{G}_i\}$. While in general $\boldsymbol{G}$ can be any full-rank tall matrix, this paper focuses on the destructive generating matrices, which are random orthogonal matrices and could prevent different kernels from sharing the representation during weight sharing. Random generating matrices empirically improve the model capacity when the model DoF is fixed. We provide an intuitive theoretical explanation of how random orthogonal matrices prevent representation sharing as follows.

For easier discussion, let us consider a special case, where all of the convolutional kernels have the same size and are used in the same shape in the corresponding convolution layers. The dimension of $\mathbf{W}$ is equal to that of one convolutional layer kernel. In other words, $\{\boldsymbol{G}_i\}$ are square matrices, and the spatial sizes of all of the convolutional kernels have the same size, $d_{in}\times d_{out}\times w\times h$, and the input channel dimension $d_{in}$ is always equal to the output channel dimension $d_{out}$. In this case, a filter $\mathbf{f}$ in a kernel can be treated as a vector in $\Re^{dwh}$. Further, we choose $\boldsymbol{G}_i$ to be a block-diagonal matrix $\boldsymbol{G}_i=\text{diag}\{\boldsymbol{A}_i,\boldsymbol{A}_i,\dots,\boldsymbol{A}_i\}$, where $\boldsymbol{A}_i\in O(dwh)$ is an orthogonal matrix that generates each filter of the kernel $\mathbf{K}_i$ from $\mathbf{W}$, and $O(\cdot)$ denotes the orthogonal group. Similar to the Proposition~2 in \cite{CheungPSP19}, we show in the Appendix \ref{sec:supp2} that: if $\boldsymbol{A}_i$, $\boldsymbol{A}_j$ are sampled from the $O(dwh)$ Haar distribution and $\mathbf{f}_i$, $\mathbf{f}_j$ are the corresponding filters (generated by $\boldsymbol{G}_i$, $\boldsymbol{G}_j$ respectively from the same set of entries of $\mathbf{W}$) from $\mathbf{K}_i$, $\mathbf{K}_j$ respectively, then we have $\E\left[\langle \mathbf{f}_i,\mathbf{f}_j \rangle\right] = 0$ and $\E \left[\langle \frac{\mathbf{f}_i}{\|\mathbf{f}_i\|},\frac{\mathbf{f}_j}{\|\mathbf{f}_j\|}\rangle^2 \right]=\frac{1}{dwh}$. Since $dwh$ is usually large, the corresponding filters from $\mathbf{K}_i$, $\mathbf{K}_j$ are close to orthogonal and generally dissimilar. This shows that even when $\{\mathbf{K}_i\}$ are generated from the same entries of $\mathbf{W}$, they are prevented from sharing the representation.

Though $\{\boldsymbol{G}_i\}$ are not updated during training, the size of $\boldsymbol{G}_i$ can be quite large in general, which can create additional computation and storage overhead. In practice, we can use permutation and element-wise random sign reflection to construct a subset of the orthogonal group as permutations and sign reflections could be implemented with high simplicity and negligible cost. A simple demonstration of $\{\boldsymbol{G}_i\}$ is demonstrated in Fig.\ref{fig:modelring_}{\bf U}\footnote{Permutations and element-wise random sign reflection conceptually are subgroups from the orthogonal group, but we shall never use them in the matrix form for the obvious efficiency purpose.}. Since pseudo-random numbers are used, it takes only two random seeds to store a random permutation and an element-wise random sign reflection.

\noindent \textbf{Even Parameter Sampling and Model Ring.}
While it is easy to randomly sample elements from $\mathbf{W}$ when generating parameters for each layer, it may not be optimal as some elements in $\mathbf{W}$ may not be evenly used, and some elements in $\mathbf{W}$ used at all due to sampling fluctuation. 
A simple equalization technique can be used to guarantee all elements of $\mathbf{W}$ are evenly sampled.
Suppose the size of $\mathbf{W}$ is $M$, and the size of parameter $\mathbf{\hat{\mathbf{W}}}$  of the model to be generated is $N$, $N>M$. As we mentioned earlier, there are $L$ layers and they require $\{\|K_1\|,\dots,\|K_L\|\}$ parameters respectively. As $N>M$, we can use $W$ as a ring: we first draw the first $\|K_1\|$ parameters from $\mathbf{\hat{\mathbf{W}}}$ followed by a pre-generated random permutation $p_1$ and a  pre-generated random element-wise sign flipping $b_1$ to construct layer-1 kernel $\mathbf{K}_1$. Then we draw the next $\|K_2\|$ parameters from $\mathbf{\hat{\mathbf{W}}}$ followed by pre-generated random permutation $p_2$ and a pre-generated random element-wise sign flipping $b_2$. We continue this process and wrap around when there is not enough entries left from $\mathbf{\hat{\mathbf{W}}}$. We refer to $\mathbf{\hat{\mathbf{W}}}$ together with this sampling strategy as \textit{model rings} since the free parameters are recurrently used in a loop. We illustrate the general parameter generator in Fig.\ref{fig:modelring_}{\bf U} and RPG in Fig.\ref{fig:modelring_}{\bf L}. This For data saving efficiency, we just need to save several random seed numbers instead of saving the pre-generated permutations $\{p_1,\dots,p_L\}$ and sign flipping operations $\{b_1,\dots,b_L\}$.

\noindent \textbf{Batch Normalization.} Model performance is relatively sensitive to the batch normalization parameters. For better performance, each convolution layer needs to have its own batch normalization parameters. In general, however, the size of batch normalization is relatively negligible. Yet when $\mathbf{W}$ is extremely small (e.g., $36$K parameters), the size of batch normalization should be considered.

\section{RPG at Multiple Scales}

We discuss the general idea of parameter generators where only one RPG is shared globally across all layers previously. 
We could also create several local RPGs, each of which is shared at certain scales, such as blocks and sub-networks. Such RPGs may be useful for certain applications such as recurrent modeling. 

\label{sec:extension}
\noindent \textbf{RPGs at Block-Level.}
Many existing network architectures reuse the same design of network blocks multiple times for higher learning capacity, as discussed in the related work. 
Instead of using one global RPG for the entire network, we could alternatively create several RPGs that are shared within certain network blocks.
We take Res18 \cite{he2016deep} as a concrete example.
Res18 has four building blocks. Every block has 2 residual convolution modules.
We create four local RPGs for Res18.
Each RPG is shared within the corresponding building block, where the size of the RPG is flexible and can be determined by users. Fig.\ref{fig:tasks}{\bf M}) illustrates how RPGs can be shared at the block-level.

\noindent \textbf{RPGs at Sub-Network-Level.} 
Reusing sub-networks, or recurrent networks, has achieved success in many tasks as they iteratively refine and improve the prediction.
Parameters are often shared when reusing the sub-networks.
This may not be optimal as sub-networks at different stages iteratively improve the prediction, and shared parameters may limit the learning capacity at different stages.
However, not sharing parameters at all greatly increases the model size.
RPG can be created for each sub-network. Such design leads to a much smaller DoF, while parameters of different sub-networks are orthogonal by undergoing destructive changes.
We show applications of sub-network-level RPGs for pose estimation and multitask regression (Section \ref{sec:pose} and \ref{sec:multi}). Fig.\ref{fig:tasks}{\bf R}) illustrates sub-network-level RPGs.
\section{Experimental Results}

\label{sec:exp}
We evaluate the performance of RPG with various tasks illustrated in Fig.\ref{fig:tasks}. 
For classification, RPG was used for the entire network except for the last fully-connected layer. We discuss performance with regard to \textit{backbone DoF}, the actual number of parameters of the backbone. For example, Res18 has $11$M backbone parameters and $512$K fc parameters, and RPG was applied to reduce $11$M backbone DoF only.

\subsection{CIFAR Classification}

\label{sec:cifar_exp}

\noindent \textbf{Implementation Details}. CIFAR experiments use 128 batch size, 5e-4 weight decay,  initial learning rate of 0.1 with gamma of 0.1 at epoch 60, 120 and 160. We use Kaiming initialization \cite{he2015delving} with adaptive scaling. Shared parameters are initialized with a particular variance and scale the parameters for each layer to make it match the Kaiming initialization.

\begin{table}[b]
\vspace{-1em}
\parbox{\linewidth}{
\vspace{-0.03in}
\caption{ \small RPG compared with multiscale deep equilibrium models (MDEQ) \cite{bai2020multiscale} on CIFAR10 and CIFAR100 classification. At the same number of model DoF, RPG achieves 3\% - 6\% performance gain with 15 - 25x less inference time. Inference time is measured by milliseconds per image.}
\label{tab:cifar-deq}
\centering
\vspace{-0.7em}
\resizebox{\linewidth}{!}{
\begin{tabular}{l|c|c|c|c}
\hline
                          &  & \multicolumn{3}{c}{ {\bf Our RPG (same DoF)} }    \\\cline{3-5} 
  \multirow{-2}{*}{ {\bf Accuracy} (\%)}  &   \multirow{-2}{*}{\bf MDEQ}                & {\bf 2x MS blk} & {\bf 3x MS blk} & {\bf 4x MS blk} \\ \hline
{\bf CIFAR10 }   & 85.1                  & 88.5         & 90.1         & \textbf{90.9}         \\ \hline
{\bf CIFAR100}   & 59.8                  & 62.8         & 64.7         & \textbf{65.7}         \\ \hline\hline
{\bf Inference time (ms)} & 3.15                 & \textbf{0.12}         & 0.18         & 0.22         \\ \hline
\end{tabular}}
}
\vfill
\vspace{0.5em}
\parbox{\linewidth}{
\centering
\caption{  \small ResNet-RPG outperforms existing DoF reduction methods \cite{han2015deep,chen2015compressing,yang2019legonet} on CIFAR100.
Additionally, a global RPG outperforms block-wise local RPGs.
}
\label{tab:cifar-block}
\vspace{-0.7em}
\resizebox{0.7\linewidth}{!}{\begin{tabular}{l|c|c}
\hline  & {\bf DoF} & {\bf Acc. (\%)} \\ \hline
{\bf R18-vanilla}             & 11M               & 77.5      \\ \hline
{\bf R34-RPG.blk} & 11M               & 78.5      \\ \hline
{\bf R34-RPG}            & 11M               & \textbf{78.9}      \\ \hline \hline
{\bf R34-random weight share}& 11M               & 74.9     \\\hline
{\bf R34-DeepCompression} \cite{han2015deep}& 11M               & 72.2    \\\hline
{\bf R34-Hash} \cite{chen2015compressing}& 11M               & 75.6     \\\hline
{\bf R34-Lego} \cite{yang2019legonet}& 11M               & 78.4     \\\hline\hline
{\bf R34-vanilla}              & 21M
& \textbf{79.1}      \\ \hline
\end{tabular}
}
}
\end{table}

\noindent \textbf{Compared to Deep Equilibrium Models}. As a representative of implicit models, deep equilibrium models \cite{bai2019deep}  reduce model DoF by finding fix points via additional optimizations. We compare the image classification accuracy on CIFAR10 and CIFAR100, as well as the inference time on CIFAR100 (Table \ref{tab:cifar-deq}). Following the settings of MDEQ \cite{bai2020multiscale}, an image was sequentially fed into the initial convolutional block, the multi-scale deep equilibrium block (dubbed as \textit{MS} block), and the classification head. 
MDEQ \cite{bai2020multiscale} achieves \textit{infinite} MS blocks by finding the fixed point of the MS block.  We reuse the MS block two to four times without increasing the model DoF. RPG achieves 3\% - 6\% gain on CIFAR10 and 3\% - 6\% gain on CIFAR100. RPG inference time is 15 - 25 times smaller than MDEQ since MDEQ needs additional time to solve equilibrium during training.

\begin{figure}
    \centering
   \includegraphics[width=\columnwidth]{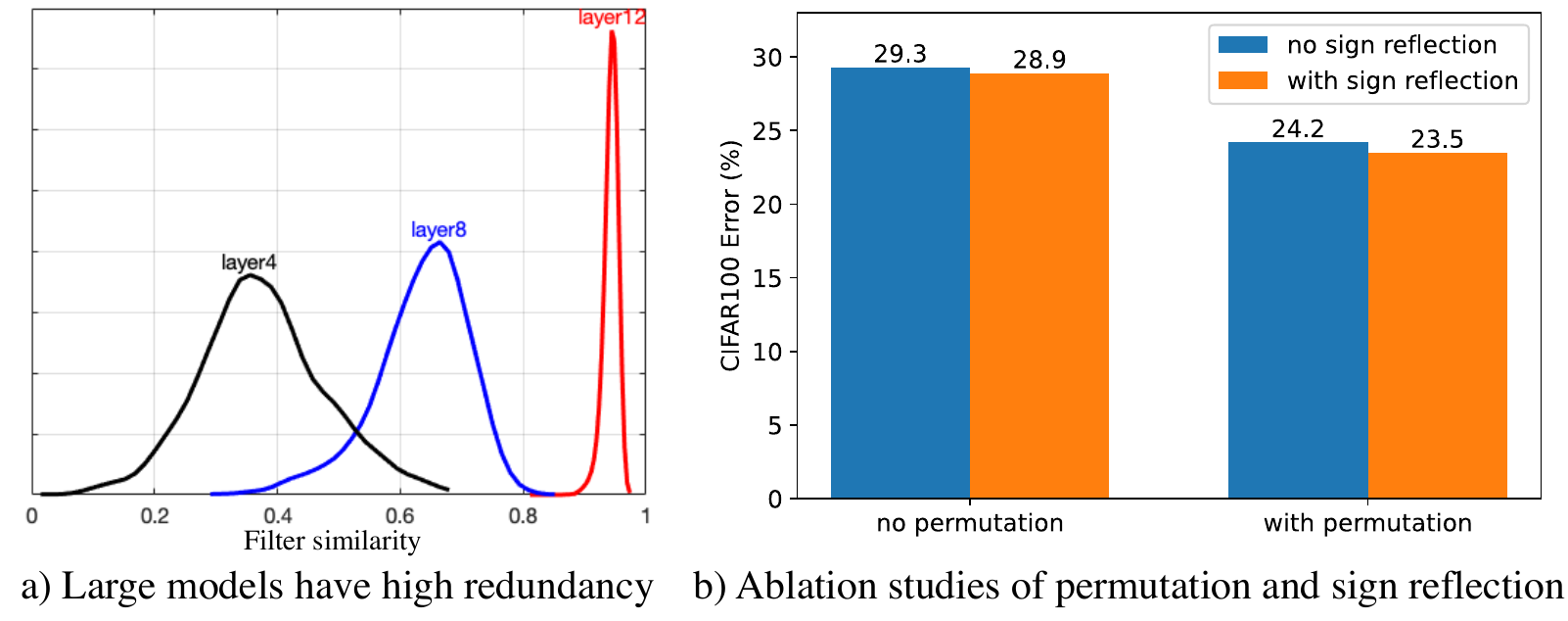}
    \caption{\small {\bf a)} Large models are known to have high redundancy and low degree of freedom (DoF). They could be pruned to small models, e.g. high filter similarity of different layers in VGG16 is observed. {\bf b)} Ablation studies of permutation and sign reflection of Res34-RPG. Having both matrices gives the highest performance.}
    \label{fig:combo1}
      \vspace{-1.5em}
\end{figure}

\noindent \textbf{Global RPG with Varying Model DoF.} 
We create one global RPG to generate parameters for convolution layers of ResNet and refer to it as \textit{ResNet-RPG}.
We report CIFAR100 top-1 accuracy of ResNet-RPG18 and ResNet-RPG34 at different model DoF (Table \ref{tab:imagenet} and Fig.\ref{fig:cifar-nparam} in Appendix \ref{sec:cifar_supp}).
Compared to ResNet, ResNet-RPG achieves higher accuracy at the same model DoF. Specifically, we achieve 36\% CIFAR100 accuracy with only 8K backbone DoF. Further, ResNet34-RPG achieves higher accuracy than ResNet18-RPG, indicating increasing time complexity gives performance gain. We observe log-linear DoF-accuracy relationship, with details in \textit{Power Law} of the following subsection.

\noindent \textbf{Local RPGs at the Block-Level}. 
In the previous Res-RPG experiments, we use one global RPG for the entire network.  
We also evaluate the performance when RPGs are shared locally at a block level, as discussed in Section \ref{sec:multi}.
In Table \ref{tab:cifar-block}, compared to plain ResNet18 at the same DoF, our block-level RPG network gives 1.0\% gain. In contrast, our ResNet-RPG (parameters are evenly distributed) gives a 1.4\% gain.
Using one global RPG where parameters of each layer are evenly distributed is 0.4\% higher than multiple RPGs. 

\noindent \textbf{Comparison to Baselines.} Table \ref{tab:cifar-block} compares RPG and other model DoF reduction methods including random weight sharing, weight sharing with the deep compression \cite{han2015deep}, hashing trick \cite{chen2015compressing} and weight sharing with Lego filters \cite{yang2019legonet}. We also compare with HyperNetworks \cite{ha2016hypernetworks} in Appendix \ref{sec:hypernet}. At the same model DoF, RPG outperforms all other baselines, demonstrating the effectiveness of the proposed method. 

\noindent \textbf{RPG for Transformers.} We apply RPG for a vision transformer ViT \cite{dosovitskiy2020image} and report results in 
Fig.\ref{fig:combo2}{\bf a}. Specifically, the ViT-tiny model with 6 transformer layers, 4 attention heads and 64 embedding dimensions, is used as a baseline. A log-linear relationship is also identified in ViT-RPG.

\subsection{ImageNet Classification}
\vspace{-0.5em}

\noindent \textbf{Implementation Details}. All ImageNet experiments use batchsize of 256, weight decay of 3e-5, and an initial learning rate of 0.3 with gamma of 0.1 every 75 epochs and 225 epochs in total. Our schedule is different from the standard schedule as the weight-sharing mechanism requires different training dynamics. We tried a few settings and found this one to be the best for RPG.

\noindent \textbf{RPG with Varying Model DoF.} We use RPG with different DoF for ResNet and report the top-1 accuracy (Table \ref{tab:imagenet} and Fig.\ref{fig:teaser}\textbf{e})). ResNet-RPGs consistently achieve higher performance than ResNets under the same model DoF. Specifically, ResNet-RPG34 achieves the same accuracy 73.4\% as ResNet34 with only half of ResNet34 backbone DoF. ResNet-RPG18 also achieves the same accuracy as ResNet18 with only half of ResNet18 backbone DoF. Further, RPG networks have higher generalizability (Section \ref{sec:general}).

\begin{table}[t!]
\centering
\LARGE
\caption{ \small ResNet-RPG consistently achieves higher performance at the same model DoF. We report ImageNet and CIFAR100 top-1 accuracy and backbone DoF for ResNet-vanilla and ResNet-RPG. }
\vspace{-0.3em}
\label{tab:imagenet}
\setlength{\tabcolsep}{20pt}
\resizebox{\columnwidth}{!}{\begin{tabular}{l|c|c|c|c|c|c|c|c}
\hline
         {\bf Acc. (\%)}      & \multicolumn{3}{c|}{ \bf R18-RPG} &  {\bf R18-vanilla} & \multicolumn{3}{c|}{ \bf R34-RPG} & {\bf R34-vanilla} \\ \hline

{\bf ImageNet }            & 40.0             & 67.2 & 70.5   & 70.5          & 41.6       & 69.1      & \textbf{73.4}   & \textbf{73.4}   \\ \hline
{\bf CIFAR100  }        &60.2              & 75.6& 77.6    & 77.6        & 61.7       & 76.5      & \textbf{78.9}   &\textbf{79.1}   \\ \hline\hline
{\bf Model DoF }          & 45K              & 2M        & 5.5M   & 11M    & 45K         & 2M        & 11M    & 21M     \\ \hline 
\end{tabular}}
\vspace{-1.5em}

\end{table}

\noindent \textbf{Power Law.} Empirically, accuracy and model DoF follow a power law, when RPG DoF is lower than 50\% ResNet-vanilla DoF (Fig.\ref{fig:teaser}{\bf d}).
The exponents of the power laws are the same for ResNet18-RPG and ResNet34-RPG on ImageNet.
The scaling law may be useful for estimating the network accuracy without training the network.
Similarly, \cite{henighan2020scaling} also identifies a power law for accuracy and model DoF of transformers.
The proposed RPG enables {\it under-parameterized} models for large-scale datasets such as ImageNet, which may unleash more new studies and findings.

\vspace{-0.5em}
\subsection{Pose Estimation}
\vspace{-0.5em}

\label{sec:pose}

\begin{table}[b]
\parbox{\linewidth}{

\centering
\LARGE
\caption{ \small RPG outperforms CPM \cite{wei2016convolutional} at the same DoF. We report pose estimation performance (model DoF) on MPII human pose compared with CPM \cite{wei2016convolutional}. The metric is PCKh@0.5. }
\label{tab:pck}
\vspace{-0.5em}
\resizebox{\linewidth}{!}{\begin{tabular}{l|c|c|c}
\hline
    {\bf Acc. (DoF)}  & {\bf CPM} \cite{wei2016convolutional}  & {\bf RPG}   & {\bf No shared w.} \\ \hline
{\bf 1x sub-net}  & \multicolumn{3}{c}{84.7 (3.3M)}         \\ \hline
{\bf 2x sub-nets} & 86.1 (3.3M) & 86.5 (3.3M) & 87.1 (6.7M)  \\ \hline
{\bf 4x sub-nets} & 86.5 (3.3M) & 87.3 (3.3M) & 88.0 (13.3M) \\ \hline
\end{tabular}}

}
\vfill
\vspace{0.3em}
\parbox{\linewidth}{
\centering
\LARGE
\caption{ \small RPG achieves the best accuracy without sharing batch normalize parameters and with permutation and sign reflection. We report multitask regression errors on S3DIS with sub-net architecture as  \cite{ramamonjisoa2019sharpnet}. 
Lower is better. All methods share the same DoF. Sub-net is reused once.}
\label{tab:multitask}
\vspace{-0.5em}
\resizebox{\linewidth}{!}{\begin{tabular}{l|c|c}
\hline
    {\bf RMSE} (\%) & {\bf Depth} & {\bf Normal} \\ \hline
{\bf Vanilla model} & 25.5     & 41.0      \\ \hline
{\bf RPG with shared BN}                                & 24.7     & 40.3      \\ \hline
{\bf Reuse \& new BN}                           & 24.0     & 39.4     \\ \hline
{\bf Reuse \& new BN \& perm. and reflect.} & \textbf{22.8}     & \textbf{39.1}      \\ \hline
\end{tabular}}
}
\end{table}
\noindent \textbf{Implementation Details.}
We superpose sub-networks for pose estimation with a globally shared RPG. 
Hourglass networks \cite{newell2016stacked} are used as the backbone.
An input image is first fed to an initial convolution block to obtain a feature map, which is then fed to multiple stacked pose estimation sub-networks. Each sub-network outputs a pose estimation prediction, which is penalized by the pose estimation loss. 
Convolutional pose machine (CPM) \cite{wei2016convolutional} share all sub-networks weights. 
We create one global RPG to generate parameters for each sub-network. Our model size is set to the same as CPM. We also compare with larger models where parameters of sub-networks are not shared.

We evaluate on MPII Human Pose dataset \cite{andriluka20142d}, a benchmark for articulated human pose estimation, which consists of over 28K training samples over 40K people with annotated body joints. We use the hourglass network \cite{newell2016stacked} as backbone and follow all their settings. 

\noindent \textbf{Results and Analysis.} We report the Percentage of Correct Key-points at 50\% threshold (PCK@0.5) of different methods in Table \ref{tab:pck}. 
CPM \cite{wei2016convolutional} share all parameters for different sub-networks. We use one RPG that is shared globally at the same size as CPM. For reference, we also compare with the no-sharing model as the performance ceiling.
Adding the number of recurrences leads to performance gain for all methods. 
At the same model size, RPG achieves higher PCK@0.5 compared to CPM. Increasing the number of parameters by not sharing sub-network parameters also leads to some performance gain.

\begin{figure*}
    \centering
   \includegraphics[width=\textwidth]{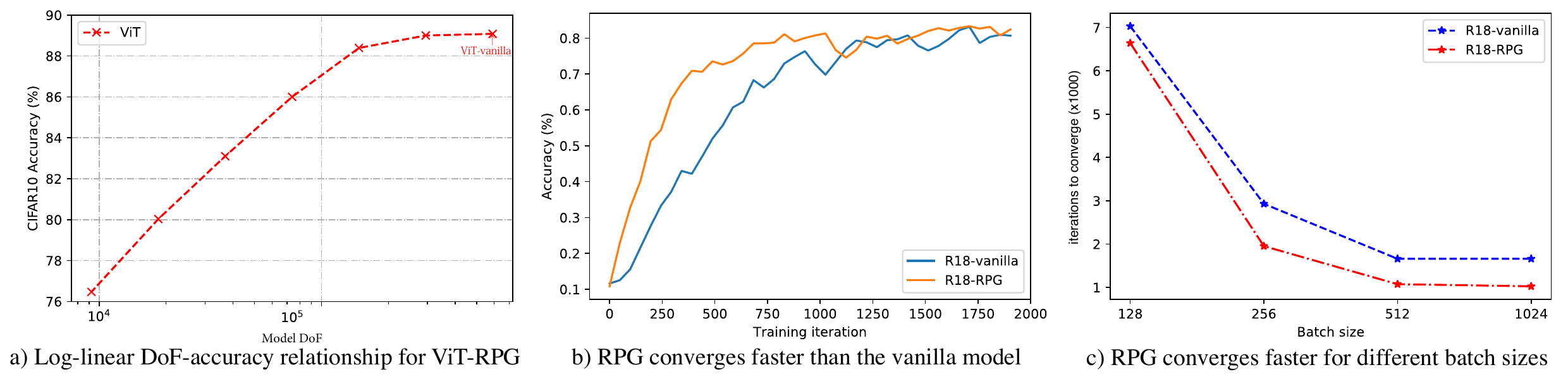}
    \caption{\small {\bf a)} A log-linear DoF-accuracy relationship exists for RPGs applied to vision transformer ViT \cite{dosovitskiy2020image}. {\bf b)} RPG converges faster than the vanilla model. We plot the CIFAR10 accuracy (smoothed by moving average) versus training iterations for Res18-vanilla and Res18-RPG. RPG converges at 1k iterations while the vanilla model converges at 1.7k. {\bf c)} RPG consistently converges faster. The reduction becomes substantial with the increasing batchsize, e.g., at batchsize 1024, RPG takes $41\%$ less iterations to converge. Denote final accuracy as $P_f$, the convergence iteration is defined when current smoothed accuracy (by moving average) is within 5\% range of $P_f$.}
    \label{fig:combo2}
\end{figure*}

\vspace{-0.5em}
\subsection{Multi-Task Regression}
\vspace{-0.5em}

\label{sec:multi}

\noindent \textbf{Implementation Details.} We superpose sub-networks for multi-task regression with multiple RPGs at the building-block level.
We  focus on predicting depth and normal maps from a given image. We stack multiple SharpNet \cite{ramamonjisoa2019sharpnet}, a network for monocular depth and normal estimation. 
Specifically, we create multiple RPGs at the SharpNet building-block level. That is, parameters of corresponding blocks of different sub-networks are generated from the same RPG.

We evaluate the monocular depth and normal prediction performance on a 3D indoor scene dataset \cite{armeni2017joint}, which contains over 70K images with corresponding depths and normals covering over 6,000 $\text{m}^2$ indoor area.
We follow all settings of SharpNet \cite{ramamonjisoa2019sharpnet}, a SOTA monocular depth and normal estimation method.

\noindent \textbf{Results and Analysis.} We report the mean square errors for depth and normal estimation in Table \ref{tab:multitask}. Compared to one-time inference without recurrence, our RPG network gives 3\% and 2\% gain for depth and normal estimation, respectively. Directly sharing weights but using new batch normalization layers decrease the performance by 1.2\% and 0.3\% for depth and normal. Sharing weights and normalization layers further decrease the performance by 0.7\% and 0.9\% for depth and normal.

\vspace{-0.5em}
\subsection{Pruning RPG}
\vspace{-0.5em}

\begin{table}[t]
\parbox{\linewidth}{
\vspace{-0.03in}
\caption{ \small RPG achieves higher post-pruning CIFAR10 accuracy and similar post-pruning accuracy drops as SOTA fine-grained pruning approach IMP \cite{frankle2019stabilizing}. Fine-grained pruning is used for reducing DoF.
}
\label{tab:fprune}
\centering
\LARGE
\vspace{-0.5em}
\resizebox{\linewidth}{!}{
\begin{tabular}{l|c|c|c|c}
\hline

                           & {\bf acc before} & {\bf acc after ↓ DoF} & {\bf acc drop} & {\bf model DoF} \\ \hline
{\bf R18-IMP} \cite{frankle2019stabilizing} & 92.3       & 90.5               & 1.8      & 274k      \\ \hline
{\bf R18-RPG} & 95.0       & 93.0               & 2.0      & 274k      \\ \hline
\end{tabular}}
}
\vfill
\vspace{0.5em}
\parbox{\linewidth}{
\centering
\LARGE
\caption{ \small RPG achieves similar post-pruning ImageNet performance as SOTA coarse-grained apporach Knapsack \cite{aflalo2020knapsack} at the same FLOPs. Coarse-grained pruning is used for reducing RPG FLOPs.
}
\label{tab:cprune}
\vspace{-0.5em}
\resizebox{\linewidth}{!}{\begin{tabular}{l|c|c|c} \hline
   & {\bf DoF before pruning} & {\bf Pruned acc.} & {\bf FLOPs}  \\ \hline
{\bf R18-Knapsack}   & 11.2M                    & 69.35\%     & 1.09e9 \\ \hline
{\bf Pruned R18-RPG} & 5.6M                     & 69.10\%     & 1.09e9 \\ \hline
\end{tabular}
}
}
\vspace{-1em}
\end{table}

\noindent \textbf{Fine-Grained Pruning}. 
Fine-grained pruning methods aim to reduce the model DoF by sparsifying weight matrices.
Such methods usually do not reduce the inference speed, although custom algorithms \cite{sgk_sc2020} may improve the speed.
At the same model DoF, RPG outperforms state-of-the-art fine-grained pruning method IMP \cite{frankle2019stabilizing}. Accuracy drops of RPG and IMP are similar, both around 2\% (Table \ref{tab:fprune}).
It is worth noting that although IMP has no run time improvement in regular settings,
it could save inference time with customized sparse GPU kernels \cite{sgk_sc2020}.

\noindent \textbf{Coarse-Grained Pruning}. While RPG is not designed to reduce FLOPs, it can be combined with coarse-grained pruning to reduce FLOPs. We prune RPG filters with the lowest $\ell_1$ norms. Table \ref{tab:cprune} shows that the pruned RPG achieves on-par performance as state-of-the-art coarse-grained pruning method Knapsack \cite{aflalo2020knapsack} at the same FLOPs.

\vspace{-0.5em}
\subsection{Analysis} 
\vspace{-0.5em}

\label{sec:general}

\begin{table*}[t]
    \caption{ \small RPG increases the model generalizability. \textbf{(a)} ResNet-RPG has lower training-validation accuracy gap on ImageNet classification. The metric is training accuracy minus validation accuracy. Lower is better. \textbf{(b)} Using RPG for pose estimation also decreases the training and validation performance GAP.  The metric is training PCK@0.5 minus validation PCK@0.5. Lower is better. \textbf{(c)} ResNet with RPG has higher performance on out-of-distribution dataset ObjectNet \cite{barbu2019objectnet}. The model is trained on ImageNet only and directly evaluated on ObjectNet. }
        \vspace{-1em}

\label{tab:general}
    \begin{subtable}{.24\linewidth}
      \centering
        \caption{IN train-val gap}
\resizebox{\textwidth}{!}{\begin{tabular}{l|c|c}
\hline
{\bf Acc gap (\%)} & {\bf vanilla} & {\bf RPG} \\ \hline
{\bf R18} & -0.7   & {\bf -2.7}        \\ \hline
{\bf R34} & 1.1    & {\bf -2.3}        \\ \hline
\end{tabular}}
    \end{subtable}%
    \hfill
    \begin{subtable}{.38\linewidth}
      \centering
        \caption{Pose train-val gap}
\resizebox{\textwidth}{!}{
\begin{tabular}{l|c|c|c}
\hline
{\bf Acc gap (\%)} & {\bf no shared w} & {\bf shared w}  & {\bf RPG} \\ \hline
{\bf 2x sub-nets} & 1.15       & 1.13    & {\bf 0.64}      \\ \hline
{\bf 4x sub-nets} & 1.98       & 1.70    & {\bf 1.15}      \\ \hline
\end{tabular}}
    \end{subtable} 
    \hfill
    \begin{subtable}{.3\linewidth}
      \centering
        \caption{OOD on ObjectNet}
\resizebox{\textwidth}{!}{
\begin{tabular}{l|c|c|c}
\hline
 & {\bf R18} & {\bf R34-RPG} & {\bf R34} \\ \hline
{\bf DoF} & 11M   & 11M        & 21M   \\ \hline
{\bf Acc.} (\%) & 13.4  & {\bf 16.5}       & 16.0  \\ \hline
\end{tabular}}
    \end{subtable}
    \vspace{-1em}
\end{table*}

\noindent \textbf{Convergence rate.} Compared with the vanilla model, RPG optimizes in a parameter subspace $\hat{\mathbf{W}} =\boldsymbol{G} \mathbf{W}$ with fewer DoF. Would such constrained optimization lead to a faster convergence rate? We analyze the convergence rate of Res18-vanilla and Res18-RPG (DoF is 5.5M, 50\% of the vanilla model) with different batchsizes. All models are trained with multi-step SGD optimizer and they all reach $>94.1\%$ final CIFAR10 accuracy. For simplicity, we analyze the first optimization stage where learning rate has not decayed.

Fig.\ref{fig:combo2}{\bf b} plots the accuracy (smoothed with moving averages) v.s. training iterations with batchsize 1024. RPG has a faster convergence rate than vanilla models. We also analyze the smoothed accuracy and identify the convergence iteration versus batchsize in Fig.\ref{fig:combo2}{\bf c}. RPG consistently converges faster than the vanilla model, and the reduction becomes substantial with the increasing batchsize.

\noindent \textbf{Comparison to Model Compression Methods}. 
 We report ResNet-RPG performance with different model DoF and existing compression methods on ImageNet (Fig.\ref{fig:teaser}\textbf{e}).
  RPG networks outperform SOTA methods such as \cite{aflalo2020knapsack,dong2019network,he2019filter,he2018soft,dong2017more,khetan2020prunenet}. For example, at the same model DoF, our RPG network has 0.6\% gain over the knapsack pruning \cite{aflalo2020knapsack}, a SOTA method of ImageNet pruning.

\noindent \textbf{Storage.} RPG models only need to save the effective parameter $\mathbf{W}$, which has the size of the model DoF, since the generation matrix $G$ is saved as a random seed at no cost. The storage space of the model file can be diminished to satisfy a smaller storage limit for inference and a faster model file transfer. Empirically on PyTorch platform, ResNet18-vanilla model file is 45MB. With no accuracy loss, ResNet18-RPG model save file size is 23MB ($ \downarrow  49\%$). With 2 percentage point accuracy loss, RPG save file size is 9.5MB  ($ \downarrow  79\%$).

\noindent \textbf{Generalizability}. 
We report the performance gap between training and validation set on ImageNet (Table \ref{tab:general}(a)) and MPII pose estimation (Table \ref{tab:general}(b)). CPM \cite{wei2016convolutional} serves as the baseline pose estimation method. RPG models consistently achieve lower gaps between training and validation sets, indicating the RPG model suffers less from over-fitting.

We also report the out-of-distribution performance of RPG models. ObjectNet \cite{barbu2019objectnet} contains 50k images with 113 classes overlapping with ImageNet. Existing models are reported to have a large performance drop on ObjectNet. 
We directly evaluate the performance of ImageNet-trained model on ObjectNet without any fine-tuning (Table \ref{tab:general}(c)).
With the same backbone DoF, R18-RPG achieves a 3\% gain compared to R18-vanilla.
With the same network architecture design, R34-RPG achieves 0.5\% gain compared to R34.
This indicates RPG networks have higher out-of-distribution performance even with smaller model DoF.

\noindent \textbf{Quantization.} Network quantization can reduce model size with minimal accuracy drop. It is of interest to study if RPG models, whose parameters have been shrunk, can be quantized. After 8-bit quantization, the accuracy of ResNet18-RPG (5.6M DoF) only drop 0.1 percentage point on ImageNet, indicating RPG can be quantized for further model size reduction. Details are in Appendix \ref{sec:quantize}. 

\noindent \textbf{Security}. Permutation matrices generated by the random seed can be considered as security keys to decode the model. Further, only random seeds to generate generating matrix $G$ need to be saved and transferred at negligible cost.

\subsection{Ablation Studies}
\vspace{-0.5em}

\label{sec:ablation}
We conduct ablation studies on CIFAR100 to analyze functions of permutation and reflection matrices (Fig.\ref{fig:combo1}{\bf b}. 
We evaluate ResNet-RPG34 with 2M backbone DoF. Permutation and sign reflection together achieves 76.5\% accuracy, while permutation only achieves 75.8\%, and sign reflection only achieves 71.1\%. Training with neither permutation nor reflection matrices achieves 70.7\%. This suggests permutation and sign reflection matrices increase RPG performance.


\section{Discussion}
\vspace{-0.5em}

 The common practice in neural network compression is to prune weights from a trained large model with many parameters or degrees of freedom (DoF).
 Our key insight is that a direct and drastically different approach might work faster and better: We start from a lean model with a small DoF, which can be linearly unpacked into a large model with many parameters. Then we can let the gradient descent automatically find the best model under the linear constraints. Our work is a departure from mainstream approaches towards model optimization and parameter reduction. We show how the model DoF and actual parameter size can be decoupled: we can define an arbitrary network of an arbitrary DoF.
 
We limit our scope to optimization with random linear constraints, termed destructive weight sharing.
However, in general, there might also exist nonlinear RPGs and efficient nonlinear generation functions to create convolutional kernels from a shared model ring $\mathbf{W}$. 
Further, although RPG focuses on reducing model DoF, it can be quantized and pruned to further reduce the FLOPs and runtime.

To sum up, we develop an efficient approach to build an arbitrarily complex neural network with any amount of DoF via a recurrent parameter generator.
On a wide range of applications, including classification, pose estimation and multitask regression, we show RPG consistently achieves higher performance at the same model DoF. Further, we show  such networks converge faster, are less likely to overfit and have higher performance on out-of-distribution data.

RPG can be added to any existing network flexibly with any amount of DoF at the user's discretion. It provides new perspectives for recurrent models, equilibrium models, and model compression. It also serves as a tool for understanding relationships between network properties and network DoF by factoring out the network architecture.

{\small
\bibliographystyle{ieee_fullname}
\bibliography{ref}

\begin{thebibliography}{10}\itemsep=-1pt

\bibitem{aflalo2020knapsack}
Yonathan Aflalo, Asaf Noy, Ming Lin, Itamar Friedman, and Lihi Zelnik.
\newblock Knapsack pruning with inner distillation.
\newblock {\em arXiv preprint arXiv:2002.08258}, 2020.

\bibitem{andriluka20142d}
Mykhaylo Andriluka, Leonid Pishchulin, Peter Gehler, and Bernt Schiele.
\newblock 2d human pose estimation: New benchmark and state of the art
  analysis.
\newblock In {\em Proceedings of the IEEE Conference on computer Vision and
  Pattern Recognition}, pages 3686--3693, 2014.

\bibitem{armeni2017joint}
Iro Armeni, Sasha Sax, Amir~R Zamir, and Silvio Savarese.
\newblock Joint 2d-3d-semantic data for indoor scene understanding.
\newblock {\em arXiv preprint arXiv:1702.01105}, 2017.

\bibitem{bai2019deep}
Shaojie Bai, J~Zico Kolter, and Vladlen Koltun.
\newblock Deep equilibrium models.
\newblock {\em Advances in Neural Information Processing Systems}, 32:690--701,
  2019.

\bibitem{bai2020multiscale}
Shaojie Bai, Vladlen Koltun, and J~Zico Kolter.
\newblock Multiscale deep equilibrium models.
\newblock {\em Advances in Neural Information Processing Systems}, 33, 2020.

\bibitem{barbu2019objectnet}
Andrei Barbu, David Mayo, Julian Alverio, William Luo, Christopher Wang, Dan
  Gutfreund, Josh Tenenbaum, and Boris Katz.
\newblock Objectnet: A large-scale bias-controlled dataset for pushing the
  limits of object recognition models.
\newblock {\em Advances in neural information processing systems},
  32:9453--9463, 2019.

\bibitem{blalock2020state}
Davis Blalock, Jose Javier~Gonzalez Ortiz, Jonathan Frankle, and John Guttag.
\newblock What is the state of neural network pruning?
\newblock In {\em Proceedings of Machine Learning and Systems}, 2020.

\bibitem{brown2020language}
Tom~B Brown, Benjamin Mann, Nick Ryder, Melanie Subbiah, Jared Kaplan, Prafulla
  Dhariwal, Arvind Neelakantan, Pranav Shyam, Girish Sastry, Amanda Askell,
  et~al.
\newblock Language models are few-shot learners.
\newblock In H. Larochelle, M. Ranzato, R. Hadsell, M.~F. Balcan, and H. Lin,
  editors, {\em Advances in Neural Information Processing Systems}, volume~33,
  pages 1877--1901. Curran Associates, Inc., 2020.

\bibitem{butko2009optimal}
Nicholas~J Butko and Javier~R Movellan.
\newblock Optimal scanning for faster object detection.
\newblock In {\em 2009 IEEE Conference on Computer Vision and Pattern
  Recognition}, pages 2751--2758. IEEE, 2009.

\bibitem{cai2018proxylessnas}
Han Cai, Ligeng Zhu, and Song Han.
\newblock Proxylessnas: Direct neural architecture search on target task and
  hardware.
\newblock In {\em International Conference on Learning Representations}, 2018.

\bibitem{carreira2016human}
Joao Carreira, Pulkit Agrawal, Katerina Fragkiadaki, and Jitendra Malik.
\newblock Human pose estimation with iterative error feedback.
\newblock In {\em Proceedings of the IEEE conference on computer vision and
  pattern recognition}, pages 4733--4742, 2016.

\bibitem{chen2015compressing}
Wenlin Chen, James Wilson, Stephen Tyree, Kilian Weinberger, and Yixin Chen.
\newblock Compressing neural networks with the hashing trick.
\newblock In {\em International conference on machine learning}, pages
  2285--2294. PMLR, 2015.

\bibitem{CheungPSP19}
Brian Cheung, Alex Terekhov, Yubei Chen, Pulkit Agrawal, and Bruno Olshausen.
\newblock Superposition of many models into one.
\newblock In {\em Advances in neural information processing systems}, 2019.

\bibitem{dollar2021fast}
Piotr Doll{\'a}r, Mannat Singh, and Ross Girshick.
\newblock Fast and accurate model scaling.
\newblock In {\em Proceedings of the IEEE/CVF Conference on Computer Vision and
  Pattern Recognition}, pages 924--932, 2021.

\bibitem{dong2017more}
Xuanyi Dong, Junshi Huang, Yi Yang, and Shuicheng Yan.
\newblock More is less: A more complicated network with less inference
  complexity.
\newblock In {\em Proceedings of the IEEE Conference on Computer Vision and
  Pattern Recognition}, pages 5840--5848, 2017.

\bibitem{dong2019network}
Xuanyi Dong and Yi Yang.
\newblock Network pruning via transformable architecture search.
\newblock In H. Wallach, H. Larochelle, A. Beygelzimer, F. d\textquotesingle
  Alch\'{e}-Buc, E. Fox, and R. Garnett, editors, {\em Advances in Neural
  Information Processing Systems}, volume~32. Curran Associates, Inc., 2019.

\bibitem{dosovitskiy2020image}
Alexey Dosovitskiy, Lucas Beyer, Alexander Kolesnikov, Dirk Weissenborn,
  Xiaohua Zhai, Thomas Unterthiner, Mostafa Dehghani, Matthias Minderer, Georg
  Heigold, Sylvain Gelly, et~al.
\newblock An image is worth 16x16 words: Transformers for image recognition at
  scale.
\newblock {\em arXiv preprint arXiv:2010.11929}, 2020.

\bibitem{frankle2019stabilizing}
Jonathan Frankle, Gintare~Karolina Dziugaite, Daniel~M Roy, and Michael Carbin.
\newblock Stabilizing the lottery ticket hypothesis.
\newblock {\em arXiv preprint arXiv:1903.01611}, 2019.

\bibitem{fung2021fixed}
Samy~Wu Fung, Howard Heaton, Qiuwei Li, Daniel McKenzie, Stanley Osher, and
  Wotao Yin.
\newblock Fixed point networks: Implicit depth models with jacobian-free
  backprop.
\newblock {\em arXiv preprint arXiv:2103.12803}, 2021.

\bibitem{sgk_sc2020}
Trevor Gale, Matei Zaharia, Cliff Young, and Erich Elsen.
\newblock Sparse {GPU} kernels for deep learning.
\newblock In {\em Proceedings of the International Conference for High
  Performance Computing, Networking, Storage and Analysis, {SC} 2020}, 2020.

\bibitem{gilbert2007brain}
Charles~D Gilbert and Mariano Sigman.
\newblock Brain states: top-down influences in sensory processing.
\newblock {\em Neuron}, 54(5):677--696, 2007.

\bibitem{ha2016hypernetworks}
David Ha, Andrew Dai, and Quoc~V Le.
\newblock Hypernetworks.
\newblock {\em arXiv preprint arXiv:1609.09106}, 2016.

\bibitem{han2015deep}
Song Han, Huizi Mao, and William~J Dally.
\newblock Deep compression: Compressing deep neural networks with pruning,
  trained quantization and huffman coding.
\newblock In {\em Proceedings of the International Conference on Learning
  Representations}, 2016.

\bibitem{hao2020multi}
Yongchang Hao, Shilin He, Wenxiang Jiao, Zhaopeng Tu, Michael Lyu, and Xing
  Wang.
\newblock Multi-task learning with shared encoder for non-autoregressive
  machine translation.
\newblock In {\em Proceedings of the 2021 Conference of the North American
  Chapter of the Association for Computational Linguistics: Human Language
  Technologies}, pages 3989--3996, 2021.

\bibitem{he2015delving}
Kaiming He, Xiangyu Zhang, Shaoqing Ren, and Jian Sun.
\newblock Delving deep into rectifiers: Surpassing human-level performance on
  imagenet classification.
\newblock In {\em Proceedings of the IEEE international conference on computer
  vision}, pages 1026--1034, 2015.

\bibitem{he2016deep}
Kaiming He, Xiangyu Zhang, Shaoqing Ren, and Jian Sun.
\newblock Deep residual learning for image recognition.
\newblock In {\em Proceedings of the IEEE conference on computer vision and
  pattern recognition}, pages 770--778, 2016.

\bibitem{he2018soft}
Yang He, Guoliang Kang, Xuanyi Dong, Yanwei Fu, and Yi Yang.
\newblock Soft filter pruning for accelerating deep convolutional neural
  networks.
\newblock In {\em Proceedings of the 27th International Joint Conference on
  Artificial Intelligence}, pages 2234--2240, 2018.

\bibitem{he2019filter}
Yang He, Ping Liu, Ziwei Wang, Zhilan Hu, and Yi Yang.
\newblock Filter pruning via geometric median for deep convolutional neural
  networks acceleration.
\newblock In {\em Proceedings of the IEEE/CVF Conference on Computer Vision and
  Pattern Recognition}, pages 4340--4349, 2019.

\bibitem{henighan2020scaling}
Tom Henighan, Jared Kaplan, Mor Katz, Mark Chen, Christopher Hesse, Jacob
  Jackson, Heewoo Jun, Tom~B Brown, Prafulla Dhariwal, Scott Gray, et~al.
\newblock Scaling laws for autoregressive generative modeling.
\newblock {\em arXiv preprint arXiv:2010.14701}, 2020.

\bibitem{howard2017mobilenets}
Andrew~G Howard, Menglong Zhu, Bo Chen, Dmitry Kalenichenko, Weijun Wang,
  Tobias Weyand, Marco Andreetto, and Hartwig Adam.
\newblock Mobilenets: Efficient convolutional neural networks for mobile vision
  applications.
\newblock {\em arXiv preprint arXiv:1704.04861}, 2017.

\bibitem{hubara2017quantized}
Itay Hubara, Matthieu Courbariaux, Daniel Soudry, Ran El-Yaniv, and Yoshua
  Bengio.
\newblock Quantized neural networks: Training neural networks with low
  precision weights and activations.
\newblock {\em The Journal of Machine Learning Research}, 18(1):6869--6898,
  2017.

\bibitem{hupe1998cortical}
JM Hup{\'e}, AC James, BR Payne, SG Lomber, P Girard, and J Bullier.
\newblock Cortical feedback improves discrimination between figure and
  background by v1, v2 and v3 neurons.
\newblock {\em Nature}, 394(6695):784--787, 1998.

\bibitem{iandola2016squeezenet}
Forrest~N Iandola, Song Han, Matthew~W Moskewicz, Khalid Ashraf, William~J
  Dally, and Kurt Keutzer.
\newblock Squeezenet: Alexnet-level accuracy with 50x fewer parameters and< 0.5
  mb model size.
\newblock {\em arXiv preprint arXiv:1602.07360}, 2016.

\bibitem{KaraletsosB20}
Theofanis Karaletsos and Thang~D. Bui.
\newblock Hierarchical gaussian process priors for bayesian neural network
  weights.
\newblock In Hugo Larochelle, Marc'Aurelio Ranzato, Raia Hadsell,
  Maria{-}Florina Balcan, and Hsuan{-}Tien Lin, editors, {\em Advances in
  Neural Information Processing Systems (NeurIPS)}, 2020.

\bibitem{karaletsos2018probabilistic}
Theofanis Karaletsos, Peter Dayan, and Zoubin Ghahramani.
\newblock Probabilistic meta-representations of neural networks.
\newblock {\em arXiv preprint arXiv:1810.00555}, 2018.

\bibitem{karpathy2015deep}
Andrej Karpathy and Li Fei-Fei.
\newblock Deep visual-semantic alignments for generating image descriptions.
\newblock In {\em Proceedings of the IEEE conference on computer vision and
  pattern recognition}, pages 3128--3137, 2015.

\bibitem{khetan2020prunenet}
Ashish Khetan and Zohar Karnin.
\newblock Prunenet: Channel pruning via global importance.
\newblock {\em arXiv preprint arXiv:2005.11282}, 2020.

\bibitem{krogh1992simple}
Anders Krogh and John~A Hertz.
\newblock A simple weight decay can improve generalization.
\newblock In {\em Advances in neural information processing systems}, pages
  950--957, 1992.

\bibitem{lecun1990optimal}
Yann LeCun, John~S Denker, and Sara~A Solla.
\newblock Optimal brain damage.
\newblock In {\em Advances in neural information processing systems}, pages
  598--605, 1990.

\bibitem{lecun1988theoretical}
Yann LeCun, D Touresky, G Hinton, and T Sejnowski.
\newblock A theoretical framework for back-propagation.
\newblock In {\em Proceedings of the 1988 connectionist models summer school},
  volume~1, pages 21--28, 1988.

\bibitem{li2016iterative}
Ke Li, Bharath Hariharan, and Jitendra Malik.
\newblock Iterative instance segmentation.
\newblock In {\em Proceedings of the IEEE conference on computer vision and
  pattern recognition}, pages 3659--3667, 2016.

\bibitem{liu2018rethinking}
Zhuang Liu, Mingjie Sun, Tinghui Zhou, Gao Huang, and Trevor Darrell.
\newblock Rethinking the value of network pruning.
\newblock In {\em International Conference on Learning Representations}, 2018.

\bibitem{louizos2018relaxed}
C Louizos, M Reisser, T Blankevoort, E Gavves, and M Welling.
\newblock Relaxed quantization for discretized neural networks.
\newblock In {\em International Conference on Learning Representations}.
  International Conference on Learning Representations, ICLR, 2019.

\bibitem{mnih2014recurrent}
Volodymyr Mnih, Nicolas Heess, Alex Graves, and Koray Kavukcuoglu.
\newblock Recurrent models of visual attention.
\newblock In {\em Advances in Neural Information Processing Systems}, 2014.

\bibitem{mozer1989using}
Michael~C Mozer and Paul Smolensky.
\newblock Using relevance to reduce network size automatically.
\newblock {\em Connection Science}, 1(1):3--16, 1989.

\bibitem{newell2016stacked}
Alejandro Newell, Kaiyu Yang, and Jia Deng.
\newblock Stacked hourglass networks for human pose estimation.
\newblock In {\em European conference on computer vision}, pages 483--499.
  Springer, 2016.

\bibitem{NowlanH92simplifying}
Steven~J. Nowlan and Geoffrey~E. Hinton.
\newblock Simplifying neural networks by soft weight-sharing.
\newblock {\em Neural Computation}, 4(4):473--493, 1992.

\bibitem{ramakrishna2014pose}
Varun Ramakrishna, Daniel Munoz, Martial Hebert, James~Andrew Bagnell, and
  Yaser Sheikh.
\newblock Pose machines: Articulated pose estimation via inference machines.
\newblock In {\em European Conference on Computer Vision}, pages 33--47.
  Springer, 2014.

\bibitem{ramamonjisoa2019sharpnet}
Michael Ramamonjisoa and Vincent Lepetit.
\newblock Sharpnet: Fast and accurate recovery of occluding contours in
  monocular depth estimation.
\newblock {\em The IEEE International Conference on Computer Vision (ICCV)
  Workshops}, 2019.

\bibitem{rastegari2016xnor}
Mohammad Rastegari, Vicente Ordonez, Joseph Redmon, and Ali Farhadi.
\newblock Xnor-net: Imagenet classification using binary convolutional neural
  networks.
\newblock In {\em European conference on computer vision}, pages 525--542.
  Springer, 2016.

\bibitem{sandler2018mobilenetv2}
Mark Sandler, Andrew Howard, Menglong Zhu, Andrey Zhmoginov, and Liang-Chieh
  Chen.
\newblock Mobilenetv2: Inverted residuals and linear bottlenecks.
\newblock In {\em Proceedings of the IEEE conference on computer vision and
  pattern recognition}, pages 4510--4520, 2018.

\bibitem{srivastava2014dropout}
Nitish Srivastava, Geoffrey Hinton, Alex Krizhevsky, Ilya Sutskever, and Ruslan
  Salakhutdinov.
\newblock Dropout: a simple way to prevent neural networks from overfitting.
\newblock {\em The journal of machine learning research}, 15(1):1929--1958,
  2014.

\bibitem{srivastava2015highway}
Rupesh~Kumar Srivastava, Klaus Greff, and J{\"u}rgen Schmidhuber.
\newblock Highway networks.
\newblock {\em arXiv preprint arXiv:1505.00387}, 2015.

\bibitem{stanley2007compositional}
Kenneth~O Stanley.
\newblock Compositional pattern producing networks: A novel abstraction of
  development.
\newblock {\em Genetic programming and evolvable machines}, 8(2):131--162,
  2007.

\bibitem{stanley2009hypercube}
Kenneth~O Stanley, David~B D'Ambrosio, and Jason Gauci.
\newblock A hypercube-based encoding for evolving large-scale neural networks.
\newblock {\em Artificial life}, 15(2):185--212, 2009.

\bibitem{szegedy2015going}
Christian Szegedy, Wei Liu, Yangqing Jia, Pierre Sermanet, Scott Reed, Dragomir
  Anguelov, Dumitru Erhan, Vincent Vanhoucke, and Andrew Rabinovich.
\newblock Going deeper with convolutions.
\newblock In {\em Proceedings of the IEEE conference on computer vision and
  pattern recognition}, pages 1--9, 2015.

\bibitem{tan2019efficientnet}
Mingxing Tan and Quoc Le.
\newblock Efficientnet: Rethinking model scaling for convolutional neural
  networks.
\newblock In {\em International Conference on Machine Learning}, pages
  6105--6114, 2019.

\bibitem{wan2020fbnetv2}
Alvin Wan, Xiaoliang Dai, Peizhao Zhang, Zijian He, Yuandong Tian, Saining Xie,
  Bichen Wu, Matthew Yu, Tao Xu, Kan Chen, et~al.
\newblock Fbnetv2: Differentiable neural architecture search for spatial and
  channel dimensions.
\newblock In {\em Proceedings of the IEEE/CVF Conference on Computer Vision and
  Pattern Recognition}, pages 12965--12974, 2020.

\bibitem{wan2013regularization}
Li Wan, Matthew Zeiler, Sixin Zhang, Yann Le~Cun, and Rob Fergus.
\newblock Regularization of neural networks using dropconnect.
\newblock In {\em International conference on machine learning}, pages
  1058--1066. PMLR, 2013.

\bibitem{wang2020orthogonal}
Jiayun Wang, Yubei Chen, Rudrasis Chakraborty, and Stella~X Yu.
\newblock Orthogonal convolutional neural networks.
\newblock In {\em Proceedings of the IEEE/CVF Conference on Computer Vision and
  Pattern Recognition}, pages 11505--11515, 2020.

\bibitem{wang2020implicit}
Tiancai Wang, Xiangyu Zhang, and Jian Sun.
\newblock Implicit feature pyramid network for object detection.
\newblock {\em arXiv preprint arXiv:2012.13563}, 2020.

\bibitem{wei2016convolutional}
Shih-En Wei, Varun Ramakrishna, Takeo Kanade, and Yaser Sheikh.
\newblock Convolutional pose machines.
\newblock In {\em Proceedings of the IEEE conference on Computer Vision and
  Pattern Recognition}, pages 4724--4732, 2016.

\bibitem{weiss2010structured}
David Weiss and Benjamin Taskar.
\newblock Structured prediction cascades.
\newblock In {\em Proceedings of the Thirteenth International Conference on
  Artificial Intelligence and Statistics}, pages 916--923. JMLR Workshop and
  Conference Proceedings, 2010.

\bibitem{wolpert1992stacked}
David~H Wolpert.
\newblock Stacked generalization.
\newblock {\em Neural networks}, 5(2):241--259, 1992.

\bibitem{wyatte2012limits}
Dean Wyatte, Tim Curran, and Randall O'Reilly.
\newblock The limits of feedforward vision: Recurrent processing promotes
  robust object recognition when objects are degraded.
\newblock {\em Journal of Cognitive Neuroscience}, 24(11):2248--2261, 2012.

\bibitem{shi2015convolutional}
Shi Xingjian, Zhourong Chen, Hao Wang, Dit-Yan Yeung, Wai-Kin Wong, and
  Wang-chun Woo.
\newblock Convolutional lstm network: A machine learning approach for
  precipitation nowcasting.
\newblock In {\em Advances in neural information processing systems}, pages
  802--810, 2015.

\bibitem{yang2019legonet}
Zhaohui Yang, Yunhe Wang, Chuanjian Liu, Hanting Chen, Chunjing Xu, Boxin Shi,
  Chao Xu, and Chang Xu.
\newblock Legonet: Efficient convolutional neural networks with lego filters.
\newblock In {\em International Conference on Machine Learning}, pages
  7005--7014. PMLR, 2019.

\bibitem{yu2018slimmable}
Jiahui Yu, Linjie Yang, Ning Xu, Jianchao Yang, and Thomas Huang.
\newblock Slimmable neural networks.
\newblock In {\em International Conference on Learning Representations}, 2018.

\bibitem{zamir2017feedback}
Amir~R Zamir, Te-Lin Wu, Lin Sun, William~B Shen, Bertram~E Shi, Jitendra
  Malik, and Silvio Savarese.
\newblock Feedback networks.
\newblock In {\em Proceedings of the IEEE conference on computer vision and
  pattern recognition}, pages 1308--1317, 2017.

\bibitem{zoph2016neural}
Barret Zoph and Quoc~V Le.
\newblock Neural architecture search with reinforcement learning.
\newblock In {\em ICLR}, 2017.

\end{thebibliography}
}
\clearpage
\newpage
\appendix

\section*{Appendices}

We first show RPG networks could be quantized with minimal accuracy drop for compression purpose in Section \ref{sec:quantize}. We then provide a figure revealing log-linear DoF-accuracy relationship in Section \ref{sec:cifar_supp}. We also provide proof for the orthogonal proposition in the main paper (Section \ref{sec:supp2}). Finally, we provide detailed comparison and discussion to a closely related work HyperNetworks \cite{ha2016hypernetworks} in Section \ref{sec:hypernet}.

Additionally, we provide the most important code to reproduce the layer superposition experiments on ImageNet in supplementary as a tgz file. The rest of code is also ready for release, and will be released after additional internal review.

\section{Quantize RPG}
\label{sec:quantize}

Quantization refers to techniques for performing computations and storing tensors at lower bitwidths than floating point precision. 
Quantization can reduce model size with tiny accuracy drop.
Table \ref{tab:quant} shows that with 8-bit quantization, ResNet18-vanilla has an accuracy drop of 0.3 percentage point, while our ResNet18-RPG has an accuracy drop of 0.1 percentage point. RPG models can be quantized for further model size reduction with a negligible accuracy drop.

\begin{table}[htb]
\centering
\LARGE
\caption{ RPG model can be quantized with very tiny accuracy drop. With 8-bit quantization on ImageNet, ResNet18-vanilla has an accuracy drop of 0.3 percentage point, while our ResNet18-RPG has an accuracy drop of 0.1 percentage point. }
\label{tab:quant}
\vspace{0.1in}
\resizebox{\columnwidth}{!}{\begin{tabular}{l|c|c|c|c}
\hline
      & {\bf \# Params} & {\bf Acc before} & {\bf Acc after ↓ quantization} & {\bf Acc drop} \\ \hline
{\bf R18-vanilla}    & 11M & 69.8       & 69.5                     & 0.3      \\ \hline
{\bf R18-RPG} & 5.6M &70.2       & 70.1                     & 0.1      \\ \hline
\end{tabular}}

\end{table}

\section{CIFAR100 Accuracy versus DoF}
\label{sec:cifar_supp}

Fig.\ref{fig:cifar-nparam} plots CIFAR100 classification accuracy versus model DoF. We observe a similar log-linear relationship as in ImageNet.

\begin{figure}[htb]\centering
\includegraphics[width=\columnwidth]{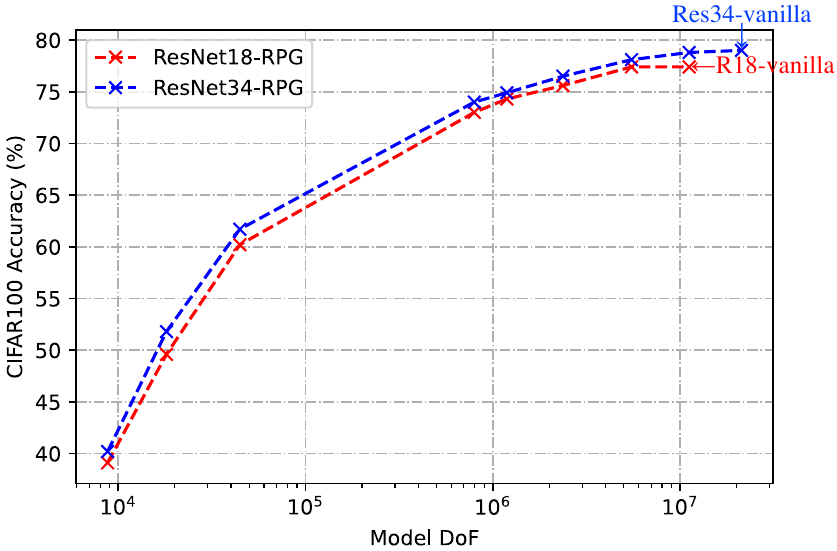}
\vspace{-1em}
\caption{ 
\small Log-linear DoF-accuracy relationship of CIFAR100 accuracy and model DoF on CIFAR100. 
RPG achieves the same accuracy as vanilla ResNet with 50\% DoF. }
\label{fig:cifar-nparam}
\end{figure}

\section{Proof to the Orthogonal Proposition}
\label{sec:supp2}
We provide proofs to the orthogonal proposition mentioned in Section \ref{sec:method:ls} of the main paper. Suppose we have two vectors $\mathbf{f}_i = \boldsymbol{A}_i \mathbf{f}, \mathbf{f}_j = \boldsymbol{A}_i \mathbf{f}$, where $\boldsymbol{A}_i$, $\boldsymbol{A}_j$ are sampled from the $O(M)$ Haar distribution.
\begin{prop}
$\E\left[\langle \mathbf{f}_i,\mathbf{f}_j \rangle\right] = 0$.
\end{prop}
\begin{proof}
\begin{align*}
    \E\left[\langle \mathbf{f}_i,\mathbf{f}_j \rangle\right] &= \E\left[\langle \mathbf{f}_i,\mathbf{f}_j \rangle\right]\\
    &= \E\left[\langle \boldsymbol{A}_i \mathbf{f},\boldsymbol{A}_j \mathbf{f} \rangle\right]\\
    &= \E\left[\langle \mathbf{f},\boldsymbol{A}_i^{T}\boldsymbol{A}_j \mathbf{f} \rangle\right]\\
    &= \mathbf{f}^T\E\left[\boldsymbol{A}_i^{T}\boldsymbol{A}_j\right]\mathbf{f}\\
    &= 0
\end{align*}
where $\boldsymbol{A}_i^{T}\boldsymbol{A}_j$ is equivalently a random sample from $O(M)$ Haar distribution and its expectation is clearly 0.
\end{proof}

\begin{prop}
$\E \left[\langle \frac{\mathbf{f}_i}{\|\mathbf{f}_i\|},\frac{\mathbf{f}_j}{\|\mathbf{f}_j\|}\rangle^2 \right]=\frac{1}{M}$.
\end{prop}
\begin{proof}
\small
\begin{align*}
    \E \left[\langle \frac{\mathbf{f}_i}{\|\mathbf{f}_i\|},\frac{\mathbf{f}_j}{\|\mathbf{f}_j\|}\rangle^2 \right] &= \frac{\E \left[\langle\boldsymbol{A}_i\mathbf{f}, \boldsymbol{A}_j\mathbf{f} \rangle^2\right]}{\|\mathbf{f}\|_2^2\|\mathbf{f}\|_2^2}\\
    &= \E\left[\langle\boldsymbol{A}\frac{\mathbf{f}}{\|\mathbf{f}\|}, \frac{\mathbf{f}}{\|\mathbf{f}\|} \rangle^2 \right], \\ &\text{where } \boldsymbol{A}=\boldsymbol{A}_i^T\boldsymbol{A}_j \sim O(M) \text{ Haar distribution}\\
    \text{ Due to the symmetr}&\text{y,}\\
    &=\E\left[\langle\boldsymbol{A}\frac{\mathbf{f}}{\|\mathbf{f}\|},(1,0,0,\dots,0)^T\rangle^2\right] \\
    \text{Let } \mathbf{g} = \boldsymbol{A}\frac{\mathbf{f}}{\|\mathbf{f}\|},\ \ \ \ \ \ \\
    &= \E\left[g_1^2\right]\\
    & = \frac{1}{M}
\end{align*}
\normalsize
since $\mathbf{g}$ is a random unit vector and $\E\left[\sum_{k=1}^{M}{g_k^2}\right] = \sum_{k=1}^{M}{\E\left[g_k^2\right]} = 1$.
\end{proof}

\section{Comparison to HyperNetworks}
\label{sec:hypernet}

HyperNetworks \cite{ha2016hypernetworks} share similarity with RPG as both methods reduce model DoF.
Specifically, HyperNetworks rely on learnable modules to generate network parameters. We compare with them and report results in Table \ref{tab:hyper}. On CIFAR100 with the embedding dimension of 64 and the same model size, HyperNetworks has 68x FLOPs as our RPG, yet 10 percentage points lower than RPG in accuracy.

\begin{table}[htb]
\centering
\LARGE
\caption{RPG outperforms HyperNetworks \cite{ha2016hypernetworks} with same DoF on CIFAR100. HyperNetworks has 68x FLOPs as our RPG, yet 10 percentage points lower than RPG in accuracy.}
\label{tab:hyper}
\vspace{0.1in}
\resizebox{\columnwidth}{!}{\begin{tabular}{l|c|c|c|c}
\hline
       & {\bf model DoF} & {\bf FLOPs} & {\bf CIFAR100 Acc.} \\ \hline
{\bf HyperNet} \cite{ha2016hypernetworks} & 632k      & 2.49G & 61.3\%        \\ \hline
{\bf RPG}      & 632k      & 36.7M & 71.6\%        \\ \hline
\end{tabular}}

\end{table}

RPG can be considered as an extreme and minimal version of HyperNetworks, one without a network. However, RPG’s unique design and implementation delivers the following advantages over HyperNetworks:
\begin{enumerate}
\item HyperNetworks add substantial FLOPs to the network and render it less practical. Given a network architecture, RPG adds minimal to no additional computation, as the permutation and sign reflection can be efficiently implemented. However, HyperNetworks use a weight generation network to generate the primary network weights. A hypernet mainly uses matrix multiplication and introduces substantial FLOPs. In the table below, we analyze FLOPs of HyperNetwork for ResNet18 with the embedding dimension of 64. FLOPs of a vanilla-Res18 for ImageNet (224 input size) and CIFAR100 (32 input size) are 1.8G and 36.7M, whereas the weight generation part of the HyperNet-Res18 takes 2.45G FLOPs. This means the weight generation FLOPs are 1.4 times of vanilla-Res18 for ImageNet and 67 times of that of CIFAR100. Empirically, we find the training and inference time HyperNet-Res18 is around 70x larger than vanilla-Res18.

\item HyperNetworks do not have an arbitrary DoF (number of reduced parameters). RPG uses a model ring of a size (model DoF) that can be arbitrarily determined. In HyperNetworks, the weight generation network uses the same hyper-weight and requires embedding to be of a certain size so that the matrix multiplication can be used for generating primary network weights. Therefore, the model DoF or reduced number of parameters cannot be arbitrarily determined. In other words, RPG decouples the model DoF (actual parameters) and the network architecture, while HyperNetworks have model DoF and architecture tightly coupled together, a highly restrictive limitation.

\item Weights generated by HyperNetworks may be coupled and not optimized for different layers. HyperNetworks use only one weight generation network parameterized by hyper-weight to generate all primary network weights. This may not be optimal as different layers of the primary network may need different weight generation networks. Additionally, matrix multiplication is used for generating weights, and the generated primary network weights may be coupled. On the other hand, RPG has destructive weight sharing, which improves the network performance by decoupling cross-layer network weights. We will add these results and discussions in the revision to clarify the differences between RPG and HyperNetworks.

\end{enumerate}

\end{document}